\documentclass[runningheads,envcountsame]{llncs}

\usepackage[T1]{fontenc}
\usepackage{graphicx}
\usepackage[english]{babel}
\usepackage{biblatex}
\usepackage{amsmath,amsfonts,mathtools}

\usepackage{amsthm}

\usepackage[colorlinks=true, allcolors=blue]{hyperref}
\usepackage{csquotes}

\usepackage[vlined,linesnumbered, ruled]{algorithm2e}
\SetKwRepeat{Do}{do}{while}%
\SetKwIF{WithProb}{ElseWithProb}{Else}{with probability}{do}{else with probability}{else}{end}

\usepackage{xcolor}

\usepackage
[
    subrefformat = simple, 
]
{subcaption}         
\usepackage{wrapfig} 

\usepackage{tikz}
\usetikzlibrary{positioning}
\usetikzlibrary{shapes.geometric,positioning}
\usepackage{cleveref}%

\newcommand{\ignore}[1]{}

\newcommand{\realnum}{\mathbb{R}}
\newcommand{\natnum}{\mathbb{N}}

\newcommand{\CalF}{\mathcal{F}}
\newcommand{\CalL}{\mathcal{L}}
\newcommand{\CalN}{\mathcal{N}}
\newcommand{\CalX}{\mathcal{X}}

\newcommand{\flip}{\CalN_\textsc{flip}}
\newcommand{\swap}{\CalN_\textsc{swap}}

\newcommand{\flipSwap}{\CalN_{f\!/\!s}}
\newcommand{\flipN}{Flip neighborhood\xspace}
\newcommand{\swapN}{Swap neighborhood\xspace}
\newcommand{\flipSwapN}{Flip/Swap neighborhood\xspace}

\newcommand{\fDS}{f^{\mathrm{DS}}}
\newcommand{\fCol}{f^{\mathrm{Col}}}

\newcommand{\set}[2]{\{ #1 \mid #2 \}}

\newcommand{\reach}{\longrightarrow}
\newcommand{\transReach}{\Longrightarrow}
\newcommand{\transReachPlus}{\overset{+}{\Longrightarrow}}

\newcommand{\reachPlus}{\overset{+}{\longrightarrow}}
\newcommand{\reachPlusStar}{\transReachPlus}

\newcommand{\reachPlusOld}
{\mathrel{\ooalign{$\longrightarrow$\cr
    \hspace{4pt}\raise1.1ex\hbox{\scriptsize $+$}\cr}}}

\newcommand{\reachPlusStarOld}{\mathrel{\ooalign{$\longrightarrow$\cr
    \hspace{4pt}\raise1.1ex\hbox{\scriptsize $+$}\cr
    \hspace{5mm}\raise1.1ex\hbox{$\ast$}\cr}}}

\usepackage{verbatim}

\title{Combinatorial Landscape Analysis for Dominating Set and Vertex Coloring}

\author{Johanna Gasse\orcidID{0009-0009-5664-3519} \and
Antonia Heinen\orcidID{0009-0002-4365-412X} \and
Felix Knöfel\orcidID{0009-0008-5341-8730} \and
Timo Kötzing\orcidID{0000-0002-1028-5228} \and
Maxim Stanko\orcidID{0009-0003-7252-5113}}
\authorrunning{J. Gasse et al.}
\institute{Hasso Plattner Institute, Potsdam, Germany}

\begin{document}

\maketitle

\begin{abstract}
We analyze the two combinatorial problems of Dominating Set and Vertex Coloring regarding what kind of local optima are present for various instances. For a variety of graph classes each, we determine whether the induced landscapes are unimodal, plateau-unimodal (all optima are just one plateau), equimodal (all local optima are global) or truly multimodal. We do this for two different neighborhood operators, one based on making only a single change and one also allowing swaps (interchanging two parts of the solution).
\end{abstract}

\section{Introduction}

\newcommand{\figureOfResultsDominating}{%
\begin{figure*}
    \bgroup
    \begin{center}
    \begin{tabular}{lccc}
        \textsc{Dominating Set}\\
        \textbf{Graph Class} & \textbf{Neighborhood} & \textbf{Optima} & \textbf{Theorem}\\ \hline
        Unbalanced Bipartite Graphs & $\flip$ & multimodal & \ref{the:ds_n1_unbal_bipartite} \\
        Complete Bipartite Graphs & $\flip$ & multimodal & \ref{the:ds_n1_complete_bipartite} \\
        Simplicial-Cluster Graphs & $\flip$ & equim., not plateau-unim. & \ref{the:ds_n1_simplicial_cluster}\\
        Non-Corona Trees & $\flip$ & multimodal & \ref{the:ds_n1_tree_corona}\\
        \hline
        Cographs & $\flipSwap$ & plateau-unim., not unim. & \ref{the:dominating_cograph} \& \ref{the:DS_negative_flipswap} \\
        Trees & $\flipSwap$ & plateau-unim., not unim.  & \ref{the:dominating_tree} \& \ref{the:DS_negative_flipswap} \\
        Interval Graphs & $\flipSwap$ & plateau-unim., not unim.  & \ref{the:dominating_interval} \& \ref{the:DS_negative_flipswap} \\
        $n$-Bouquet of 6-Cycles & $\flipSwap$ & multim. & \ref{the:ds_n2_C6_bouquet} 
    \end{tabular}
    \end{center}
    \egroup
    \captionsetup{type=table}
    \caption{
    Classification of the landscapes of various graph classes for the Dominating Set problem into plateau-unimodal (there exists a single local optimum plateau), equimodal (all local optimum plateaus are globally optimal) and multimodal (there exist NGLOPs, see \Cref{def:plateau}) for the Flip and \flipSwapN respectively (see \Cref{def:flip} and \Cref{def:flipswap}). See \Cref{def:modal} for the definition of modality.
    }
    \label{tab:results_domSet}
\end{figure*}}

\newcommand{\figureOfResultsColoring}{%
\begin{figure*}
    \bgroup
    \begin{center}
    \begin{tabular}{lccc}
        \textsc{Vertex Coloring}\\
        \textbf{Graph Class} & \textbf{Neighborhood} & \textbf{Optima} & \textbf{Theorem}\\ \hline
        Universal Bipartite Graphs & $\flip$  & unimodal & \ref{the:universal}\\
        $C_{6n}$ & $\flip$  & multimodal & \ref{the:3-circle}\\
        Crown Graphs & $\flip$ & multimodal & \ref{the:crown}\\
        Chordal Graphs & $\flip$ & equim., not plateau-unim. & \ref{the:coloring_chordal} \\
        Chordal Bipartite Graphs & $\flip$ & unimodal & \ref{the:coloring_chordal_bipartite}\\
        \hline
        $C_{6n}$  & $\flipSwap$  & unimodal & \ref{the:c6kSwap}\\
        Crown Graphs & $\flipSwap$  & unimodal & \ref{the:crownSwap}\\
        Spoked $C_{12n}$ & $\flipSwap$ & multimodal & \ref{the:spoked12k}\\
    \end{tabular}
    \end{center}
    \egroup
    \captionsetup{type=table}
    \caption{
    Classification of the landscapes of various graph classes for the Vertex Coloring problem into plateau-unimodal (there exists a single local optimum plateau), equimodal (all local optimum plateaus are globally optimal) and multimodal (there exist NGLOPs, see \Cref{def:plateau}) for the Flip and \flipSwapN respectively (see \Cref{def:flip} and \Cref{def:flipswap}). See \Cref{def:modal} for the definition of modality.
    }
    \label{tab:results_coloring}
\end{figure*}}

A lot of theory on randomized search heuristics is focused on deriving run time bounds for given algorithms on given problems. These analyses come to a limit when the search space is too complex to allow such a careful analysis, particularly when local optima are present. Many interesting areas of application for randomized search heuristics are for problems with such complex characteristics, leading naturally to a divide between what problems can get a rigorous run time analysis and what problems search heuristics are actually applied to.

With this paper, we extend our knowledge to larger problem classes by relaxing the goal. Instead of finding bounds on the run time of a given algorithm, we explore problem characteristics in terms of local optima for two combinatorial problems. This approach was taken experimentally for the number partitioning problem, the binary knapsack problem and the quadratic binary knapsack problem in~\cite{alyahya_landscape_2019} (see also~\cite{DBLP:journals/tec/Prugel-BennettT12,DBLP:journals/swevo/Tayarani-Najaran15} for similar experimental analyses). In this paper we derive mathematical results for the combinatorial problems of Dominating Set and Vertex Coloring.

Concretely, we are interested in the existence of a non-global local optimum (NGLO). Note that the definition of a local optimum is tricky when neighboring search points can have equal fitness, leading to plateaus. A plateau where all search points neighboring the plateau have worse fitness we call a local optimum plateau, giving then the notion of a non-global local optimum plateau (NGLOP). We distinguish the following four cases for landscapes: (a) unimodal (only one global optimum, no NGLOPs); (b) plateau-unimodal (there can be a single plateau of local optima); (c) equimodal (all local optimum plateaus have equal fitness, equivalent to there being no NGLOPs); and (d) multimodal (there are NGLOPs). See \Cref{sec:locOpt} for precise definitions.

The absence of NGLOPs informs us about the fact that a simple hill climbing algorithm (eventually) finds a global optimum. It is well known that even unimodal search spaces can be hard to optimize, so our results do not inform directly about problem difficulty in terms of run time. However, the results illustrate
\begin{itemize}
    \item which problem classes suffer from local optima and thus require non-local operators;
    \item which problem classes admit hill climbing in principle, and could thus be targeted by further analysis.
\end{itemize}
For example, on instances where we show that even a complex set of instances does not have NGLOPs, a run time analysis could now continue and make a finer analysis of these instances, knowing that there are no fundamental road blocks, just potentially large detours.

Finally, our analyses have the advantage that they clearly expose where problem difficulty comes from, while large run times may be attributed to local optima, long paths, large plateaus or other reasons.

We consider the two combinatorial problems of Dominating Set and Vertex Coloring. Both are problems on simple graphs $G=(V,E)$ with $n = |V|$  vertices. Note that all graphs in this paper are assumed to be connected and simple. For the remainder of this introduction, we give an overview of our results for these two structurally rather different problems. In \Cref{sec:landscapes_general}, we define landscapes and all related terms. \Cref{sec:domSet} and \Cref{sec:vertexColoring} give our results for Dominating Set and Vertex Coloring respectively. Note that some results are derived from \emph{reconfiguration theory} (see, for example, \cite{nishimura_introduction_2018} for an introduction).

\subsection{Overview of Results for Dominating Set}
 
The optimization task for Dominating Set is to find a cardinality-minimal set $S \subseteq V$ such that every vertex is either in $S$ or has a neighbor in $S$. We consider the following two neighborhoods, corresponding to local search operators. The \flipN either adds or removes a vertex from a given set; we denote this neighborhood by $\flip$. The \flipSwapN either removes a vertex, adds a vertex, or does both (swap); we denote this neighborhood by $\flipSwap$.

We consider different classes of graphs and determine, for each graph in the given class, what the induced landscape is in terms of local optima. An overview of our results is given in \Cref{tab:results_domSet}. As we can see, for the \flipN, even very simple graph classes, such as complete bipartite graphs, induce multimodal landscapes (except for finitely many such graphs). We identify only a small graph class, containing certain trees, that induce an equimodal landscape (in fact, for trees we give a full characterization of when they induce an equimodal or multimodal landscape, see \Cref{the:ds_n1_tree_corona}). This shows that Dominating Set is very hard for hill climbers, when only single bit flips are permitted.

    \figureOfResultsDominating

From \Cref{tab:results_domSet} we see that a lot of graph classes induce plateau-unimodal landscapes when using the \flipSwapN; we complement these positive findings with an example graph class which, while simple, induces multimodal landscapes.

\subsection{Overview of Results for Vertex Coloring}
 
The optimization task for Vertex Coloring is to find a coloring of all vertices of the given graph using as few as possible colors while maintaining that, for each edge, the colors of the endpoints of the edge are colored differently. We first consider the neighborhood of recoloring a single vertex ($\flip$). As a second operator we either recolor a single vertex or interchange the colors of two vertices ($\flipSwap$). Note that we use the same names for the operators as for Dominating Set, since we make a general definition which covers both cases simultaneously (see \Cref{sec:landscapes_general}).

    \figureOfResultsColoring

We again consider different classes of graphs. An overview of our results is given in \Cref{tab:results_coloring}. As we can see, certain types of bipartite graphs induce unimodal landscapes (up to renaming colors) even under the \flipN. Furthermore, the large class of chordal graphs is also amenable to analysis, giving an equimodal landscape.

Turning to the \flipSwapN, we see that here also crown graphs and certain cycles induce unimodal landscapes. We construct a simple example graph class based on cycles with certain chords added to showcase which (simple) graphs induce multimodal landscapes.

\section{Introduction to Landscapes}
\label{sec:landscapes_general}

In this section we formally introduce the notion of a landscape. For all $n \in \natnum$ we use the notation $[n] = \{1,\ldots,n\}$. Just as in~\cite{alyahya_landscape_2019} we model fitness landscapes as triples as follows. Note that these definitions assume minimization.

\begin{definition}
    A directed graph $(\CalX,\CalN)$ consisting of set $\CalX$ of \emph{individuals}, a set $\CalN$ of directed edges is called a \emph{neighborhood graph}, with $\CalN$ being the \emph{neighborhood}. For $x \in \CalX$ we write $\CalN(x) = \set{y \in \CalX}{(x,y) \in \CalN)}$ for the neighbors of $x$. For a \emph{fitness function} $f\colon \CalX \rightarrow \realnum$, we call $(\CalX,\CalN,f)$ a \emph{landscape}.

    For a given landscape $(\CalX,\CalN, f)$ and $x,y \in \CalX$, we write $x \reach y$ if $(x,y) \in \CalN$; we write $x \transReach y$ if it is possible to reach $y$ from $x$ in any number of such steps, i.e.\ if there is a path from $x$ to $y$ in the graph $(\CalX,\CalN)$.
    We write $x \reachPlus y$ if $(x,y) \in \CalN$ and $f(x) \geq f(y)$; we write $x \transReachPlus y$ if it is possible to reach $y$ from $x$ in any number of such steps, i.e.\ if there are $z_1,\ldots,z_k$ such that $$x \reachPlus z_1 \reachPlus z_2 \reachPlus \ldots \reachPlus z_k \reachPlus y.$$ We refer to such a sequence $x, z_1, \dots, z_k, y$ as a \emph{positive path}.
\end{definition}

Note that both $\transReach$ and $\transReachPlus$ are transitive relations.

For this work we focus on only individuals which are fixed-length strings of a bounded set of natural numbers; most common are fixed-length bit strings.

\begin{definition}
Let $n \in \natnum$ and $r \in \natnum_{\geq 2}$. We define $\CalX_r = \{0,\ldots,r-1\}^n$.
\end{definition}

Note that we can encode many problems with this search space. For example, to denote any subset of $n$ given options (vertices of a graph, \ldots), we can use bit strings ($\CalX_2$), where a $1$ indicates that the corresponding option is chosen, and a $0$ indicates that it is not chosen.

For this specific set of individuals we define the following classic neighborhoods

\begin{definition} \label{def:flip}
    Let $r \in \natnum_{\geq 2}$. We define the neighborhood $\flip$ on $\CalX_r$ such that, for all $u, v \in \CalX_r$, $(u, v) \in \flip$ if and only if $u$ can be transformed into $v$ by changing a single position, i.e.\ iff $u \neq v$ and there is $i \in [n]$ such that for all $j \in [n] \setminus \{i\}$ we have $u(j) = v(j)$. We call $\flip$ the \emph{\flipN}.
\end{definition}

\begin{definition} \label{def:swap}
    Let $r \in \natnum_{\geq 2}$. We define the neighborhood $\swap$ on $\CalX_r$ such that, for all vertices $u, v \in \CalX_r$, $(u, v) \in \swap$ if and only if $u$ can be transformed into $v$ by swapping two positions, i.e.\ iff there are $i,j \in [n]$ such that for all $k \in [n] \setminus \{i,j\}$ we have $u(k) = v(k)$, $u(i) = v(j)$ and $u(j) = v(i)$. We call $\swap$ the \emph{\swapN}.
\end{definition}

\begin{definition} \label{def:flipswap}
    Let $r \in \natnum_{\geq 2}$. We define the neighborhood $\flipSwap$ on $\CalX_r$ such that, $\flipSwap = \flip \cup \swap$, so if either a flip or a swap transformation is allowed. We call $\flipSwap$ the \emph{\flipSwapN}.
\end{definition}

\subsection{Local Optima}\label{sec:locOpt}

In this paper, we consider minimization. Note that all definitions carry over to maximization.

We seek to analyze the landscape for the existence of local optima, meaning, intuitively, individuals where all neighboring individuals have a worse fitness. We are especially interested in non-global local optima. We now give a formal definition for the different types of local optima.

\begin{definition} [Local and Global Optimum] \label{def:optimum}
    Let $x$ be an individual and $(\CalX, \CalN, f)$ a landscape. We call $x$ a \emph{local optimum} if, for all individuals $y$ with $(x,y) \in \CalN$, we have $f(x) < f(y)$. Further, we call $x$ a \emph{global optimum} if, for all individuals $z \in \CalX$, $f(z) \ge f(x)$. A local optimum which is not a global optimum we call an \emph{NGLO} (non-global local optimum).
\end{definition}

Now we define the term \emph{plateau} to refer to a set of neighboring individuals such that all of them have the same fitness. 

\begin{definition} [Plateau] \label{def:plateau}
    Let $(\CalX, \CalN, f)$ be a landscape and let $P \subset \CalX$ be a set of individuals such that, for all $x, y \in P$, $f(x) = f(y)$ and $P$ is connected in the graph $(\CalX,\CalN)$. Then we call $P$ a \emph{plateau}.  We call a plateau $P$ a \emph{local optimum plateau} if, for all individuals $y \notin P$ reachable in one step from an individual $x \in P$, it holds that $f(x) < f(y)$. A local optimum plateau $P$ where the elements are not global optima is called an \emph{NGLOP} (non-global local optimum plateau).
\end{definition}

Note that any local optimum, taken as a singleton set, is also a local optimum plateau. This means that in proving no NGLOPs exist, we also show that no NGLOs exist. 

\begin{definition}[Modality] \label{def:modal}
    Let $\CalL = (\CalX, \CalN, f)$ be a landscape. We call $\CalL$
    \begin{itemize}
        \item \emph{unimodal}  if there is exactly one local optimum plateau, and it has size $1$;
        \item \emph{plateau-unimodal} if there is exactly one local optimum plateau;
        \item \emph{equimodal} if the elements of all local optimum plateaus have the same fitness;
        \item \emph{multimodal} if there are two or more elements of local optimum plateaus with different fitness. 
    \end{itemize}
\end{definition}

Note that any unimodal landscape is plateau-unimodal (but not vice versa) and any plateau-unimodal landscape is equimodal (but not vice versa). Furthermore, a landscape is multimodal if and only if it is not equimodal.

\begin{definition}[Modality and Instances] \label{def:modal_inst}
    We apply the terms unimodal, pla\-teau-unimodal and equimodal to a set of instances if all landscapes induced by an instance of the given set  have that property.
    The term multimodal we apply to a set of instances if all but finitely many\footnote{We exclude finitely many instances to allow for small instances in the set, which are typically too small to be multimodal.} of these instances induce a multimodal landscape.
\end{definition}

\section{Dominating Set}

\label{sec:domSet}

In this section we regard the \flipN and the \flipSwapN in turn for the Dominating Set problem. For this, we define the fitness as follows.

\begin{definition}
    Given a graph $G = (V,E)$ on $n = |V|$ vertices, we fix an ordering of the vertices and identify each vertex with its index from $[n]$. Furthermore, we identify bit strings $x \in \CalX_2$ with the corresponding sets $S = \set{i \in [n]}{x_i = 1}$. We let $\fDS_G$ be the fitness function induced by $G$ regarding the dominating set problem, defined with penalty terms for each undominated vertex as follows.
    \begin{align*}
    \fDS_G\colon \CalX_2 &\rightarrow \realnum_{\geq 0}, \\
    x &\mapsto \sum_{i=1}^n x_i + (n+1) \cdot |\{i \in [n] \mid x_i = 0 \wedge \forall j \in [n]\colon \{i,j\} \in E \rightarrow x_j = 0\}|.
    \end{align*}
\end{definition}

\subsection{Dominating Set with the Flip Neighborhood}

We start by considering different graph classes and analyzing, for given graphs $G$ from these classes, the induced landscape for the \flipN $(\CalX_2, \flip, \fDS_G)$.

We start with the following proposition, establishing that local optima in that landscape are exactly the $\subseteq$-minimal dominating sets. We will use this in the following without further mention.

\begin{proposition}\label{lem:lo_equals_mds}
    Let $G=(V,E)$ be a graph and $S \subseteq V$. Then $S$ is an $\subseteq$-minimal dominating set if and only if $S$ is a local optimum in $(\CalX_2, \flip, \fDS_G)$.
\end{proposition}
\begin{proof}
    Suppose first that $S$ is not a dominating set. Adding an undominated vertex to $S$ will decrease the fitness by at least $(n + 1) - 1 = n$ (the cardinality of the set increases by $1$, but at least one vertex not dominated by $S$ is now dominated). Thus, $S$ is not a local optimum and not an $\subseteq$-minimal dominating set, giving the equivalence in this case.

    Suppose now $S$ is a dominating set; we prove the directions separately.
    
    \smallskip
    \noindent
    ($\Rightarrow$) We show the contrapositive and assume that $S$ is not $\subseteq$-minimal. Then there exists a vertex $v \in S$ such that $S' := S \setminus \{v\}$ is still a dominating set. Because $S'$ differs from $S$ by one vertex and
    $\fDS_G(S') = |S'| = |S| - 1 < |S| = \fDS_G(S)$,
    we get that $S$ is not a local optimum.
    
    \smallskip
    \noindent
    ($\Leftarrow$) Conversely, suppose that $S$ is a minimal dominating set. We show that $S$ is a local optimum.
    Consider any neighbor $S' \in \CalN_1(S)$. There are two possible cases:
    
    \begin{itemize}
        \item If $S' = S \setminus \{v\}$ for some $v \in S$, then by minimality of $S$, the set $S'$ is no longer a dominating set and therefore infeasible. Hence $\fDS_G(S')$ is strictly worse than $\fDS_G(S)$.
        \item If $S' = S \cup \{v\}$ for some $v \notin S$, then
        $\fDS_G(S') = |S| + 1 > |S| = \fDS_G(S).$
    \end{itemize}
    In both cases, no neighbor has strictly smaller objective value than $S$. Therefore, $S$ is a local optimum.
\end{proof}

We start with some general considerations about landscapes from Dominating Set instances.

\begin{lemma}\label{lem:ds_n1_plateau_is_singleton}
    Let $G$ be a graph. In the fitness landscape $(\CalX_2, \flip, \fDS_G)$, every plateau $P \subseteq \CalX$ is a singleton set.
\end{lemma}
\begin{proof}
    Let $S$ and $S'$ be any two solutions with $(S,S') \in \flip$. By definition, $S'$ is obtained by either adding or removing exactly one vertex from $S$. 
    Under the objective function $f_{DS}$, a 1-flip change necessarily results in $\fDS_G(S) \neq \fDS_G(S')$: even if the number of undominated vertices stays the same, the cardinality of $S'$ will be higher that that of $S$ by $1$ when adding a vertex and lower by $1$ when removing a vertex. Since no two adjacent solutions have the same objective value, no two distinct solutions can belong to the same plateau. It follows that every plateau $P$ must be a singleton set.
\end{proof}

From this, we immediately get the following corollary.

\begin{corollary}\label{cor:ds_n1_no_plateau_unimodal}
     Let $G$ be a graph and consider landscape $(\CalX_2, \flip, \fDS_G)$.
     The property of plateau-unimodality is equivalent to unimodality. Consequently, a landscape with multiple local optima cannot be plateau-unimodal; it is \emph{equimodal} if all local optima share the same fitness value, and \emph{multimodal} otherwise.
\end{corollary}

The first graph class we consider are bipartite graphs. The following lemma is central to our analyses.

\begin{lemma}\label{lem:bipartite_partitions_are_mds}
    Let $G=(U \uplus V, E)$ be a connected bipartite graph. Then the sets $U$ and $V$ are local optima for the landscape $(\CalX_2, \flip, \fDS_G)$.
\end{lemma}
\begin{proof}
    Consider the set $S=V$. Since $G$ is connected, every vertex in $U$ has a neighbor in $V$, so $S$ dominates $U$. Since $S$ also dominates itself, it is a dominating set. Furthermore, since $G$ is bipartite, $S$ is an independent set. Therefore, no vertex $v \in S$ is dominated by $S \setminus \{v\}$, meaning no vertex can be removed without losing the domination property. Thus, $S$ is a minimal dominating set and as such a local optimum. By symmetry, the same holds for $S'=U$.
\end{proof}

Note that it is possible that neither of the partitions is a global optimum; for example, in the path with six vertices $P_6$ there is a dominating set of size two, while both partitions have size three.

\begin{theorem}\label{the:ds_n1_unbal_bipartite}
    Let $G=(U\uplus V, E)$ be an unbalanced bipartite graph ($|U|\neq|V|$, assume w.l.o.g.\ that $|U| < |V|$). Then the fitness landscape $(\CalX_2, \flip, \fDS_G)$ is multimodal.
\end{theorem}
\begin{proof}
    By \Cref{lem:bipartite_partitions_are_mds}, $U$ and $V$ are local optima. Since $|V|>|U|$, $V$ cannot be a global optimum and must be an NGLO. Thus $\CalL$ is multimodal.
\end{proof}

\begin{theorem} \label{the:ds_n1_complete_bipartite}
    The graph class of all complete bipartite graphs induces multimodal landscapes.
\end{theorem}
\begin{proof}
    Let $G=(U \uplus V, E)$ be a complete bipartite graph. We suppose $|U| > 2$ or $|V| > 2$, excluding finitely many instances.
    
    If the graph is unbalanced (i.e., $|U| \neq |V|$), the statement holds directly by \Cref{the:ds_n1_unbal_bipartite}. Now consider the balanced case where $|U|=|V|> 2$. By \Cref{lem:bipartite_partitions_are_mds}, $U$ is a local optimum. Consider the set $S = \{u, v\}$ for any $u \in U$ and $v \in V$. $S$ is a dominating set because $u$ dominates all vertices in $V$ and $v$ dominates all vertices in $U$. Since $|S| = 2 < |U|$, $U$ is an \emph{NGLO}, and $\CalL$ is multimodal.
\end{proof}

The next graph class we consider is the class of all trees. We start by proving a result on more general graphs before using this to characterize the landscape induced by trees.

\begin{definition}[Corona Graphs]
    A connected graph $G$ is called a \emph{Corona graph} if every vertex of $G$ is either a leaf (that is, has degree of $1$) or is adjacent to exactly one leaf. A \emph{Corona tree} is a Corona graph that is also a tree. 
\end{definition}

The terminology derives from the \emph{Corona product}, introduced by \cite{frucht_corona_1970}. Here, the \emph{Corona product} $G \circ H$ of two graphs $G$ and $H$ is obtained by taking one copy of $G$ and $|V(G)|$ copies of $H$, and then joining the $i$-th vertex of $G$ to every vertex in the $i$-th copy of $H$. This specific class of trees has been studied under this name by \cite{barik_spectrum_2007} and \cite{panda_unicyclic_2017}. Specifically, a Corona tree is isomorphic to the graph $H \circ K_1$, where $H$ is a tree.

We generalize this definition of Corona graphs as follows.

\begin{definition}[Simplicial-Cluster Graphs]
    A graph $G$ is a \textbf{simplicial-cluster graph} if its vertex set can be partitioned into $n$ cliques $L_1, \ldots, L_k$ of size $|L_i| \ge 2$, such that each clique $L_i$ contains a vertex $q_i$ (the \textbf{simplicial anchor}) that is adjacent to no vertex outside of $L_i$. 
    
    \medskip\noindent
    \textbf{Remark.}
    The simplicial anchors act as ``private'' leaves for each cluster. This class generalizes the Corona graph, which corresponds to $|L_i| = 2$ for all $i$. In particular, Corona trees are obtained when the clusters are connected only via bridges.
\end{definition}

\begin{theorem}\label{the:ds_n1_simplicial_cluster}
    Let $G$ be a simplicial-cluster graph. The landscape $(\CalX_2,\flip,\fDS_G)$ is equimodal, but not plateau-unimodal. 
\end{theorem}

\begin{proof}
    Let $G$ be simplicial-cluster graph, and let $L_1,\dots,L_k$ denote the clusters each containing a respective simplicial anchor $q_1,\ldots,q_k$. 

    Consider any local optimum $S\in\CalX$, i.e., any minimal dominating set. It must contain at least one vertex from every cluster $L_i$,  otherwise $q_i$ would be undominated. Conversely, $S$ cannot contain more than one vertex from the cluster: since each cluster is a clique, including multiple vertices would violate minimality, since now any vertex within the cluster can be removed while preserving domination.

    Therefore, $S$ contains exactly one vertex from each cluster $L_i$.  As a result, all local optima of $G$ have the same fitness. Hence, $\CalL$ is equimodal.

    Since there are at least two available vertices for $S$ from every cluster $L_i$, there are at least $2^k\geq 2$ local optima. By \Cref{cor:ds_n1_no_plateau_unimodal}, $\CalL$ is not plateau-unimodal.
\end{proof}

We use this result to get a full characterization of the landscapes induced by arbitrary trees.

\begin{theorem}\label{the:ds_n1_tree_corona}
    Let $G$ be a tree. The landscape $\CalL = (\CalX_2, \flip, \fDS_G)$ is equimodal if $G$ is a Corona tree, and multimodal if $G$ is \emph{not} a Corona tree. 
\end{theorem}

\begin{proof}
    Suppose first that $G$ is a Corona tree, then it is in particular a simplicial-cluster graph. By \Cref{the:ds_n1_simplicial_cluster}, $\CalL$ is equimodal.
    
    Assume now that $G$ is \emph{not} a Corona tree. Then there exists a vertex $v$ such that either
    \begin{enumerate}
        \item $v$ is adjacent to at least two leaves, or
        \item $v$ is a non-leaf vertex with no leaf neighbors.
    \end{enumerate}
    
    \medskip
    \noindent
    \emph{Case 1: There exists a vertex $v \in V$ having at least two leaf neighbors.}
    Let $v$ be a vertex, and let $L(v)=\{\ell_1,\dots,\ell_k\}$ with $k\ge 2$ denote the set of leaf neighbors of $v$. 
    Since $G$ is bipartite, fix a bipartition $(U,V)$ of $G$ such that $v\in U$.
    By \Cref{lem:bipartite_partitions_are_mds}, the set $U$ is a minimal dominating set of~$G$. Every non-leaf neighbor of $v$ has degree at least $2$ and is therefore adjacent to some vertex in $U$ other than $v$. Thus, each such neighbor is dominated by at least one vertex of $U\setminus\{v\}$. Now define
    $$
        D := U \quad\text{and}\quad
        D' := (U \setminus \{v\}) \cup L(v).
    $$
    The set $D'$ is a dominating set of $G$: each leaf $\ell_i$ dominates itself and $v$, while all other vertices remain dominated as in $U$. Moreover, $D'$ is minimal. Indeed, no leaf $\ell_i$ can be removed without leaving $\ell_i$ undominated, and removing any other vertex destroys domination exactly as in $U$.
    
    Since $|D'| = |U| + (k-1)$ with $k\ge 2$, we have $|D'| > |D|$. Hence, $D'$ is a strictly worse local optimum, and $\CalL$ is multimodal.

    \medskip
    \noindent
    \emph{Case 2: There exists a vertex $v \in V$ which is not a leaf and which has no leaf neighbors.}
    In this case, $v$ is adjacent only to non-leaf vertices.
    
    By \Cref{the:ds_n1_unbal_bipartite}, we may assume that $G$ is a balanced bipartite graph, since otherwise the claim already holds. Let $(U,W)$ be a bipartition of $G$ and let $V:=U\cup W$. Remove $v$ from $G$ and consider the forest 
    $F := G - v.$
    Since $G$ is a tree and $v$ is not a leaf, $F$ consists of $f\ge 2$ connected components $F_1,\dots,F_f$, each of which is a tree. For each component $F_i$, let $(U_i,W_i)$ denote its bipartition, labeled such that $|U_i|\le |W_i|$.
    By \Cref{lem:bipartite_partitions_are_mds}, both $U_i$ and $W_i$ are minimal
    dominating sets of $F_i$.
    
    Because $G$ is balanced, $|V|$ is even, and hence $|F|=|V|-1$ is odd. Therefore, at least one component $F_i$ has an odd number of vertices, and for such a component we have $|U_i|<|W_i|$. Call a component \emph{unbalanced} if $|U_i|<|W_i|$, and let $o\ge 1$ denote the number of unbalanced components of $F$.
    
    We now define two sets $D_F$ and $D'_F$ composed of vertices from the forest $F$. For each component $F_i$, the set $D_F$ contains $U_i$ if $F_i$ is unbalanced, and otherwise contains the bipartition class that dominates $v$. Analogously, the set $D'_F$ contains $V_i$ for each unbalanced component $F_i$, and for balanced components contains the bipartition class that does not dominate $v$. From these base sets, we define two dominating sets of $G$: let $D = D_F \cup \{v\}$ if $v$ is not dominated by $D_F$, and $D = D_F$ otherwise. Define $D'$ similarly analogously based on $D'_F$.
    
    Both $D$ and $D'$ dominate $G$. Moreover, both are minimal dominating sets: removing any vertex selected from a component $F_i$ destroys domination in that component, while removing $v$ (if present) leaves $v$ undominated.
    
    We now compare the cardinalities. For each balanced component $F_i$, the contributions to $|D_F|$ and $|D'_F|$ are equal. For each unbalanced component $F_i$, the contribution of $F_i$ to $|D'_F|$ is strictly larger than its contribution to $|D_F|$. Therefore, $|D'_F| \ge |D_F| + o$.
    
    If $v\notin D$, then $|D| = |D_F|$. Using the inequality above, 
    $$|D'| \ge |D'_F| \ge |D_F| + o = |D| + o.$$
    Since $o \ge 1$, we have $|D'|>|D|$.
    
    If $v\in D$, then $v$ is not dominated by $D_F$. This occurs only if $G$ contains no balanced components (as their selection in $D_F$ ensures domination of $v$) and the neighbors of $v$ in all unbalanced components belong to the larger sets $V_i$. Thus, all components are unbalanced, implying $o=f\ge 2$. In this case, 
    $$|D'| \ge |D'_F| \ge |D_F| + o = (|D|-1) + o.$$ 
    Since $o \ge 2$, we have $|D'| \ge |D| + 1$, and again $|D'|>|D|$. Thus, $G$ admits minimal dominating sets of different cardinalities, and the larger one is a non-global local optimum.
\end{proof}

\subsection{Dominating Set and the \flipSwapN}
We will use results from \emph{reconfiguration theory} in order to find results for the \flipSwapN for some elaborate graph classes. For this we need the following definition.

\begin{definition}[Reconfiguration Sequence]
    Let $G$ be a graph on $n$ vertices. Let $A, B$ be two dominating sets of $G$ and $k \in \natnum$. Then we write $A \rightsquigarrow_k B$ if there is a sequence of dominating sets transforming $A$ to $B$ such that two successive sets differ by either adding or removing a vertex and such that the cardinality of any set in the sequence is bounded from above by $k$. We call the corresponding sequence of dominating sets a \emph{reconfiguration sequence}.
\end{definition}

We give the connection to landscapes in the following lemma.

\begin{lemma}\label{lem:reconfiguration_sequence_to_positive_path}
    Let $G$ be a graph on $n$ vertices and $k \in \natnum$. 
    Let $A, B$ be two dominating sets of $G$ with $\lvert A \rvert \ge \lvert B \rvert$ such that $A \rightsquigarrow_k B$. 
    
    If a reconfiguration sequence between $A$ and $B$ exists such that no two token additions happen in direct succession, it holds that $A \reachPlusStar B$ for $(\CalX_2,\flipSwap, \fDS_G)$. 
\end{lemma}
\begin{proof}
Let $A, B$ be two solutions such that $A \rightsquigarrow_k B$ using TAR and $\lvert A \rvert \ge \lvert B \rvert$.
  
We begin by noting several observations on the relationship between reachability in reconfiguration and fitness neighborhoods.
    \begin{itemize} 
        \item As we use TAR as a reconfiguration rule, it is impossible for two solutions adjacent in the reconfiguration graph to have the same cardinality (and thus, the same fitness): The cardinality must always either increase or decrease by one.
        \item Token removal followed by token addition (and vice versa) corresponds to a swap operator. Therefore, for any two solutions $A$ and $B$, if $A$ can be transformed into $B$ by one token addition and one removal, it holds that $A \reachPlus B$ for the \swapN.
        \item Token removal improves the fitness, as it decreases the parameter to be minimized while maintaining feasibility. Therefore, for any two solutions $A$ and $B$, if $A$ can be transformed into $B$ by one token removal, it holds that $A \reachPlus B$ for the \flipN.
        \item Token addition worsens the fitness, as it increases the parameter to be minimized. Therefore, for any two solutions $A$ and $B$, if $A$ can be transformed into $B$ by one token addition, $A \reachPlus B$ does not hold for the \flipN.
    \end{itemize}  

Now, let $(R_1, \dots, R_n)$ be a reconfiguration sequence between $A$ and $B$ such that no two token additions happen in direct succession. We first note that we also know there are at least as many token removals as additions in $R$, as otherwise, it would violate the assumption that $\lvert A \rvert \ge \lvert B \rvert$.

As noted above, we know that any token addition can be paired up with a token removal and thus be represented by a balanced operator removing one vertex from the set and adding another. Therefore, for any $(R_i, R_{i+1}, R_{i+2})$ such that one token addition and one token removal are performed, we know that $R_i \reachPlus R_{i+2}$ for the \swapN. 

We now argue that it is always possible to pair up a token addition with a token removal. We pair up token additions with token removals from first to last, using the following rule: If there is a token removal directly preceding the token addition that has not been paired with another token addition, pair with that. Otherwise, pair with the token removal directly following the token addition.

We now prove this rule always yields a pairing by contradiction. Assume there was a token addition that could not be paired with a token removal. That specifically means that there was no token removal following it. As we know that no two token additions happen in direct succession, this means that the token addition must have been the final reconfiguration in the reconfiguration sequence. Therefore, it is reconfiguration $(R_{n-1}, R_n)$. It must hold that reconfiguration $(R_{n-2}, R_{n-1})$ is a token removal or does not exist (if $n < 2$), as no two token additions directly succeed each other. Further, $(R_{n-3}, R_{n-2})$ must then be a token addition and paired with $(R_{n-2}, R_{n-1})$, as otherwise, it would have been chosen to be paired with $(R_{n-1}, R_n)$. The same logic can be applied until the first reconfiguration $(R_1, R_2)$ is reached, which must be a token addition (as any token removal would have to be preceded by a token addition with which it is paired). However, this means that one more token addition than token removal was performed (as we both start and end with a token addition and otherwise always alternate between the two), which violates the assumption that there may be no more token additions than removals. Therefore, it must always be possible to pair up a token addition with a removal.

By replacing any pairings of token addition and removal, we have only balanced operators with $R_i \reachPlus R_{i+2}$ for the \swapN and token removals (for which $R_i \reachPlus R_{i+1}$ holds for the \flipN) holds. Therefore, $A \reachPlusStar B$ holds for the \flipSwapN, proving the theorem.
\end{proof}

\citeauthor{haddadan_complexity_2015} \cite{haddadan_complexity_2015} investigate the properties of reconfiguration on dominating sets, introducing the concept of \emph{canonical dominating sets}.

\begin{definition} [Canonical Dominating Set] \label{def:canonical_dominating_set}A minimum dominating set $C$ for a graph $G$ is called a \emph{canonical dominating set} if, for every dominating set $D$, it holds that $D \rightsquigarrow_k C$ with $k=\lvert D \rvert+1$.
\end{definition}

They then prove the existence of canonical dominating sets for trees, cographs and interval graphs. We argue that this implies plateau-unimodality of the landscapes induced by such graphs.

\begin{theorem} \label{the:canonical_dominating_sets}
    Let $G$ be a graph and suppose $G$ contains a canonical dominating set $C$. Regarding the landscape $(\CalX_2,\flipSwap, \fDS_G)$, it holds for all feasible dominating sets $D$ that $D \transReachPlus C$.
\end{theorem}

\begin{proof}
    Let $m= \lvert C \rvert$ be the cardinality of a minimum dominating set for $G$. Let $i \ge m$ and let $D$ be a dominating set of cardinality $i$. We now prove that $D \transReachPlus C$ by induction over the solution cardinality $i$ and using \Cref{lem:reconfiguration_sequence_to_positive_path}.

    Let $i=m$ and let $D$ be any dominating set of cardinality $i$. Then, by the definition of a canonical dominating set, it holds that $D \rightsquigarrow_{i+1} C$. Let $R=(R_1, \dots, R_j)$ be a reconfiguration sequence from $D$ to $C$. 
    
    We prove by contradiction that no two token additions happen in direct succession. Assume two token additions happen in succession. We know that the dominating set before the first of two such addition must have been of cardinality at least $m$, as $m$ is the smallest cardinality possible for a feasible dominating set. Therefore, after the token additions, the cardinality of the current dominating set is strictly greater than $m+1$, which contradicts the condition that no dominating set on the path may have a cardinality greater than $i+1=m+1$. Therefore, no two token addition happen in direct succession.
    Using this and the fact that $\lvert D \rvert = m = \lvert C \rvert$, it holds by \Cref{lem:reconfiguration_sequence_to_positive_path} that $D \transReachPlus C$ for the \flipSwapN.

    Now, let $i \geq m$ be such that, for all dominating sets $D$ of cardinality $i$, $D \transReachPlus C$.
    We now prove that the same holds for all dominating sets of size $i+1$. Let $D$ be a dominating set of cardinality $i+1$. Then, by the definition of a canonical set, it holds that $D \rightsquigarrow_{i+2} C$. Let $R=(R_1, \dots, R_j)$ be a reconfiguration sequence from $D$ to $C$ and let $D'$ be the first dominating set of cardinality $i$ in $R$, which must exist assuming $\lvert D \rvert > m$. Let $R'=(R_1, \dots, R_{j'})$ be the reconfiguration sequence from $D$ to $D'$. 

    We again prove by contradiction that no two token additions happen in direct succession. Assume two token additions happen in succession. We know that the dominating set before the first addition must have been of cardinality greater $i$, as we have not yet reached $D'$, which is the first dominating set of cardinality $i$ on the path. Therefore, after the token additions, the dominating set's cardinality is greater than $i+2$, which contradicts the condition that no dominating set on the path may have a cardinality greater than $i+2$. Therefore, this is not possible.

    Using this and the fact that $\lvert D \rvert = i+1 > i = \lvert D' \rvert$, it holds by \Cref{lem:reconfiguration_sequence_to_positive_path} that $D \transReachPlus D'$ for the \flipSwapN. Further, by the induction hypothesis, we know that $D' \transReachPlus C$, which implies $D \transReachPlus C$, proving the theorem.
\end{proof}

Using this theorem, we can now conclude that cographs, trees and interval graphs are plateau-unimodal.

\begin{definition}[Cograph]
    A \emph{cograph}, or complement-reducible graph, or $P_4$-free graph, is a graph that can be generated from the single-vertex graph $K_1$ by complementation and disjoint union.
\end{definition}

Examples of cographs are complete graphs, complete bipartite graphs and threshold graphs.

\begin{theorem}\label{the:dominating_cograph}
    Let $G$ be a cograph. Then the fitness landscape $(\CalX_2,\flipSwap, \fDS_G)$ is plateau-unimodal.
\end{theorem}
\begin{proof}
    From~\cite[Lemma~5]{haddadan_complexity_2015} we know that any cograph admits a canonical dominating set $C$. By  \Cref{the:canonical_dominating_sets}, it follows that for any dominating set $D$, $D \transReachPlus C$ in the \flipSwapN. Thus, no NGLOPs exist and, for all global optima, a positive path between them exists.
\end{proof}

\begin{theorem}\label{the:dominating_tree}
    Let $G$ be a tree. Then the fitness landscape $(\CalX_2,\flipSwap, \fDS_G)$ is plateau-unimodal.
\end{theorem}
\begin{proof}
    From~\cite[Lemma~7]{haddadan_complexity_2015} we know that any tree admits a canonical dominating set $C$. By  \Cref{the:canonical_dominating_sets}, it follows that, for any dominating set $D$, $D \transReachPlus C$ in the \flipSwapN. Hence, no NGLOPs exist and, for all global optima, a positive path between them exists.
\end{proof}

\begin{definition}[Interval Graph]
    A graph $G=(V, E)$ with $V = \{v_1, v_2, \dots , v_n\}$ is an \emph{interval graph} if there exists a set $I$ of (closed) intervals $I_1, I_2, \dots , I_n \subseteq \realnum$ such that, for all $i,j \in [n]$, $\{v_i, v_j\} \in E$ if and only if $I_i \cap I_j \neq \emptyset$. 
\end{definition}

\begin{theorem}\label{the:dominating_interval}
    If $G$ is an interval graph, then the fitness landscape $(\CalX_2,\flipSwap, \fDS_G)$ is plateau-unimodal.
\end{theorem}
\begin{proof}
    From~\cite[Lemma~11]{haddadan_complexity_2015} we know that any interval graph admits a canonical dominating set $C$. By  \Cref{the:canonical_dominating_sets}, it follows that, for any dominating set $D$, $D \transReachPlus C$ in the \flipSwapN. It follows that no NGLOPs exist and, for all global optima, a positive path between them exists.
\end{proof}

We now turn to negative results for the \flipSwapN; note that these are not connected to reconfiguration theory.

We have seen that the class of all cographs is plateau-unimodal for the \flipSwapN. The following example shows that it is not unimodal, also for trees and interval graphs.

\begin{theorem}\label{the:DS_negative_flipswap}
    Let $k \in \natnum_+$. For the graph $G = P_{3k+2}$, which is a tree, a cograph, and an interval graph, we have that $(\CalX_2, \flipSwap, \fDS_G)$ is not unimodal.
\end{theorem}
\begin{proof}
    Consider the graph $P_{3k+2} = (V,E)$ with $V = \{1,\ldots,3k+2\}$.

    Consider the vertex sets $D_1 := \set{3i+1}{i \in [k]}$ and $D_2 := \set{3i+2}{i \in [k]}$. Note that both $D_1$ and $D_2$ are minimal dominating sets. Since $D_1 \neq D_2$, $(\CalX_2, \flip, \fDS_{P_{3k+2}})$ is not unimodal.
\end{proof}

Finally, we give an example graph class which is multimodal for the \flipSwapN.

\begin{definition}[$n$-bouquet]
    An \emph{$n$-bouquet of $6$-cycles} is a graph constructed by joining $n \ge 2$ copies of the cycle graph $C_6$ at a single shared vertex $v_0$, i.e., by performing a 1-clique-sum of these cycles.  
    For the $i$-th cycle, we denote the remaining vertices by $v_{i,1}, v_{i,2}, v_{i,3}, v_{i,4}, v_{i,5}$, where $v_{i,1}$ and $v_{i,5}$ are adjacent to the shared vertex $v_0$, and consecutive vertices are adjacent in the natural order of the cycle.
\end{definition}
\begin{theorem}\label{the:ds_n2_C6_bouquet}
    Let $G=(V,E)$ be an \emph{$n$-bouquet of $6$-cycles}. Then the landscape $\CalL=(\CalX_, \flipSwap, \fDS_G)$ is multimodal.
\end{theorem}
\begin{proof}

    Consider the solution 
    $$S = \{ v_0 \} \cup \{ v_{i,3} \mid i = 1, \dots, n \},$$
    so that $|S| = n+1$.  
    The set $S$ is a dominating set: the shared vertex $v_0$ dominates all $v_{i,1}$ and $v_{i,5}$, while each $v_{i,3}$ dominates $v_{i,2}$, $v_{i,3}$, and $v_{i,4}$ of its respective cycle. Notice that every vertex in the graph is dominated by exactly one vertex in $S$, so the domination regions do not overlap.  
    
    As a consequence, removing any vertex from $S$ would destroy the domination property, and no vertex can be swapped in to replace an existing one, since no other vertex fully dominates the region of any $v_{i,3}$ or of $v_0$. Therefore, $S$ is a local optimum of size $n+1$.

    Now consider the solution
    $$S' = \{ v_{i,2}, v_{i,5} \mid i = 1, \dots, n \},$$
    so that $|S'| = 2n$. 

    The set $S'$ is a dominating set: for every cycle $i$, the vertices $v_{i,1}$ and $v_{i,3}$ are dominated by $v_{i,2}$, while $v_{i,4}$ and the shared vertex $v_0$ are dominated by $v_{i,5}$. Hence all vertices of the graph are dominated.

    For every cycle, it is possible to exchange $v_{i,5}$ for $v_{i,4}$ and $v_{i,2}$ for $v_{i,1}$. These exchanges work in both directions. However, for each cycle, at least one of $\{v_{i,2}, v_{i,1}\}$ must be contained in the set in order to dominate $v_{i,3}$, and in at least one cycle one of $\{v_{i,5}, v_{i,4}\}$ must be contained in the set to dominate $v_0$.

    It is never possible to swap a vertex for $v_0$. 
    Swapping $v_{i,1}$ would leave $v_{i,2}$ undominated, since in this configuration $v_{i,4}$ must be the second selected vertex of the cycle. 
    Swapping $v_{i,2}$ would leave $v_{i,2}$ itself undominated, as $v_{i,4}$ or $v_{i,5}$ must then be the second selected vertex in the cycle. 
    By symmetry, the same holds for $v_{i,4}$ and $v_{i,5}$.
    
    It is also never possible to swap a vertex for $v_{i,3}$. 
    Swapping $v_{i,1}$ or $v_{i,2}$ would leave $v_{i,1}$ undominated, and by symmetry the same argument applies to $v_{i,4}$ and $v_{i,5}$.
    
    Finally, removing any vertex from $S'$ leaves that vertex itself undominated, since in this configuration no two selected vertices are adjacent. Thus, $S'$ lies on a local optimum plateau.
    
    Since $n \ge 2$, we have $|S| = n+1 < 2n = |S'|$. Thus, $\CalL$ contains local optima with different fitness and is therefore multimodal.
\end{proof}

\section{Vertex Coloring}
\label{sec:vertexColoring}

In this section we consider the vertex coloring problem with the \flipN, then with the \flipSwapN. For this, we define the fitness as follows.

\begin{definition}
    Given a graph $G = (V,E)$ on $n = |V|$ vertices, we fix an ordering of the vertices and identify each vertex with its index from $[n]$. Any string $x \in \CalX_n$ is a coloring of the vertices of $G$. We call a coloring proper if, for all $i,j \in V$ with $\{i,j\} \in E$ we have $x_i \neq x_j$. We let $\fCol_G$ be the fitness function induced by $G$ regarding the vertex coloring problem, defined as the number of colors used with a penalty terms for each monochromatic edge as follows.
    $$
    \fCol_G\colon \CalX_n \rightarrow \realnum_{\geq 0}, x \mapsto |\set{x_i}{i\in [n]}| + (n+1) \cdot |\set{\{i,j\} \in E}{x_i = x_j}|.
    $$
\end{definition}

Note that, for any infeasible coloring, recoloring any vertex which is part of a monochromatic edge with a color that does not currently appear in the coloring will improve fitness. Thus, there are no local optima among infeasible solutions, and we will focus on feasible solutions.

\subsection{\flipN}

We first make the following observation.

\begin{proposition} \label{prop:uni_equi}
    On connected bipartite graphs, any fitness landscape that is equimodal is also unimodal, assuming that two optimal colorings are considered the same if they differ only in the naming of colors. 
\end{proposition}

This follows from the fact that for connected bipartite graphs, there is exactly one feasible $2$-coloring, namely each partition colored in a unique color. While the names of those colors might differ between colors (one coloring might use red and green, the other blue and red), they are still essentially the same coloring and we therefore consider them the same for the sake of unimodality.

We now seek to further characterize the fitness landscape for coloring, focusing on the family of bipartite graphs. Here, we identify one additional unimodal case as well as two multimodal cases. We start with a definition.

\begin{definition} [Universal Bipartite Graphs] \label{def_universal}
    Let $G=(V, E)$ be a bipartite graph with partitions $U \uplus W = V$. We call $u \in U$ \emph{universal to the opposite partition} if, for all $w \in W$, it holds that $\{u, w\} \in E$. Analogously we define universality of $w \in W$. A bipartite graph with at least one vertex universal to the opposite partition is called a \emph{universal bipartite graph}.
\end{definition}

Next we state that the existence of a vertex that is universal to the opposite partition implies that the fitness landscape $(\CalX_n, \flip, \fCol_G)$ is unimodal.

\begin{theorem} \label{the:universal}
   For a universal bipartite graph $G$, the landscape $(\CalX_n, \flip, \fCol_G)$ is unimodal.
\end{theorem}

\begin{proof}
Let $G=(V,E)$ be a bipartite graph with partitions $U \uplus W = V$.
Without loss of generality, let $u \in U$ be a vertex that is universal to the opposite partition. 

Let $x$ be a feasible coloring that contains at least three colors. We now prove that $x$ is not part of an NGLOP by showing the existence of a positive path from $x$ to a coloring containing only two colors. In doing so, we prove that $(\CalX_n, \flip, \fCol_G)$ is equimodal, making it unimodal by \Cref{prop:uni_equi}.

As $x$ is feasible, we know that it contains no monochromatic edges. Therefore, for any given vertex $w \in W$, we have $x_u \ne x_w$ as we know that $\{u, w\}\in E$. This means the color $x_u$ does not appear anywhere in the partition $W$.

We now first recolor, one by one starting from $x$, the vertices in $U$ to the color $x_u$, resulting in a coloring $y$. Since all vertices of $W$ do not have the color $x_u$, all intermediate colorings and $y$ are feasible, and we have $x \transReachPlus y$.

Let $c$ be a color used by $y$ in $W$. We note that this color does not appear in $U$, which is now colored in a single color. Now we recolor, one by one starting from $y$, the vertices in $W$ to the color $c$, resulting in a coloring $z$. Since all vertices of $U$ are colored $x_u \neq c$, all intermediate colorings and $z$ are feasible, and we have $y \transReachPlus z$. Using transitivity, we have a positive path from $x$ to a feasible $2$-coloring. Since $x$ had more than $2$ colors, it is not a local optimum.

\end{proof}

Now, we examine structural properties that allow the existence of NGLOPs, meaning the fitness landscape $(\CalX_n, \flip, \fCol_G)$ is multimodal. 

First, we regard $G = C_{6k}$. We find the fitness landscape $(\CalX_n, \flip, \fCol_G)$  to contain NGLOs, making it multimodal. Further, we define conditions for a graph $G$ containing the $C_{6k}$ as a subgraph that lead to $(\CalX_n, \flip, \fCol_G)$ containing NGLOPs, making it likewise multimodal.

We examine a situation in which an NGLOP exists at three colors. We find that the $C_{6k}$ contains NGLOPs at three colors. In order to describe a graph configuration with the $C_{6k}$ as an induced subgraph that likewise contains NGLOPs, we first give necessary definitions.

We refer to vertices in the $C_{6k}$ as $v_1, \dots, v_{6k}$, with the numbering corresponding to their order in the circle.

\begin{definition} [$3$-chordless $C_{6k}$-subgraph] \label{def:3co}
    Let $G=(V, E)$ be a graph and $v_1, v_2, \dots v_{6k} \in V$ a $C_{6k}$-subgraph in $G$. We call $v_1, v_2, \dots v_{6k} \in V$ \emph{$3$-chordless} if, for all $j \le 6k$, there exists no $l \le 6k$ such that $l \equiv j \pmod 3$ and $\{v_j, v_l \}\in E$.
\end{definition}

Further, we define the term \emph{$d$-opposing vertices} to describe a way vertices outside of the $C_{6k}$ can be connected to vertices within.

\begin{definition} [$d$-opposing vertices]
   Let $G=(V, E)$ be a graph and let $v_1, \dots v_{6k} \in V$ be a $C_{6k}$-subgraph. Let $v \in V$ be a vertex that is not part of the $C_{6k}$ and let $d \in \{ 0, 1, 2 \}$. We call $v$ \emph{$d$-opposing} if, for all $j\le 6k$ such that $j \equiv d \pmod 3$, no edge exists between $v$ and $v_j$.
\end{definition}

Next we prove that a bipartite graph contains an NGLOP at three colors provided that it contains a $3$-chordless $C_{6k}$-subgraph and certain criteria for the other vertices are met. To prove this, we first show that the $C_{6k}$ is multimodal with an NGLO at three colors for the fitness landscape $(\CalX_n, \flip, \fCol_G)$ and then prove that such a coloring of the $C_{6k}$-subgraph is likewise possible in a graph meeting the given criteria.

\begin{theorem} \label{the:3-circle}
    Let $G=(V,E)$ be a bipartite graph such that 
    \begin{enumerate}
    \item it contains a $3$-chordless $C_{6k}$-subgraph,
    \item for all $v \in V$ outside of the $C_{6k}$-subgraph, there exists a $d \in \{0,1,2 \}$ such that $v$ is $d$-opposing and
    \item if two vertices $u, v \in V$ are both $d$-opposing, then $\{u,v\} \notin E$.
\end{enumerate}
Then the fitness landscape $(\CalX_n, \flip, \fCol_G)$ for $G$ is multimodal with NGLOPs at three colors.
\end{theorem}
\begin{proof}
To prove this, we first construct a feasible $3$-coloring $x$ and then show that it is an NGLOP. 

First, let $v_1, v_2, \ldots v_{6k} \in V$ be a $3$-chordless $C_{6k}$ in $G$ for which the above criteria are met. For all $j\le i$, let $x_{v_j} \equiv j \pmod 3$. Such a coloring is possible as each vertex appears only once, meaning we assign only one color per vertex. 

Now, we show that the coloring is feasible. First, note that $i \equiv 0 \pmod 3$, so the first and last vertex are not colored in the same color (as the first vertex is colored in color $1$), and for all $j\le 6k$, it holds that $j \not\equiv j+1 \pmod 3$. This means that no two adjacent vertices in the circle are colored in the same color. Further, we know that for all $j \le 6k$, there exists no $l \le 6k$ such that $l \equiv j \pmod 3$ and $\{v_j, v_l \}\in E$, which implies that no two vertices of the same color are connected by an edge, meaning the coloring is feasible.

Now, it remains to be shown that we can color the remaining vertices in such a way that the coloring remains feasible. For this, let $v\in V$ be a non-$C_{6k}$ vertex. Let $d \in \{0,1,2\}$ such that $v$ is $d$-opposing, which we know must be the case for at least one value $d$. Then, let $x_v=x_{v_{d+1}}$. As $v$ is $d$-opposing, we know that it is not connected to any same-colored vertex within the $C_{6k}$.

Let $u$ be a vertex outside of the $C_{6k}$ such that $x_v=x_u$. We know that, in order for $u$ to be colored in $x_v$, it must be $d$-opposing. Then, it follows from the condition that no edge can exist between $v$ and $u$, making the coloring feasible.

Now, we show that $x$ is an NGLOP. For this, we show that the coloring of the $C_{6k}$-subgraph is an NGLO and that changes to the coloring in the rest of the graph have no bearing on the fitness as long as feasibility is maintained, making it a plateau.

We note that $x$ has three colors in use, meaning that 
\begin{enumerate}
    \item it is a non-optimal coloring, as a coloring with two colors exists and
    \item any neighboring coloring using more colors would immediately have a worse fitness, meaning that we can focus only on neighboring colorings using three or fewer colors.
\end{enumerate}

First, we consider the coloring within the $C_{6k}$. Let $j \le i$ and $y$ be a coloring neighboring $x$ such that $x[v_j] \ne y[v_j]$. We know that both vertices neighboring $v_j$ (meaning $v_{j-1}$ and $v_{j+1}$) have different colors in $x$ and therefore also in $y$, as the color of $v_j$ is the only one that changed. Therefore, assuming the total number of colors does not increase, it must either hold that $x_{v_j} = x_{v_{j-1}}$ or $x_{v_j} = x_{v_{j+1}}$. Either of these options would lead to an infeasible coloring, thereby worsening the fitness. With this, it is shown that the coloring within the $C_{6k}$-subgraph is an NGLO.

As we now know that any change to a color within the $C_{6k}$-subgraph worsens the fitness of the coloring and there are three colors in use within the $C_{6k}$-subgraph, we also know that we cannot decrease the number of colors used below three. Since the coloring is feasible already, the only way to improve the fitness would be to reduce the number of colors used, meaning that no neighboring coloring where a vertex outside of the $C_{6k}$ changes colors can improve the fitness. Therefore, $x$ is part of an NGLOP, as there might be neighboring colorings with the same fitness (where vertices outside of the $C_{6k}$ have different colors), but none with a better fitness.
\end{proof}

We now generalize those findings to a class of graphs whose fitness landscape is multimodal with NGLOPs at $k$ colors.

\begin{definition} [$k$-Crown Graph]
    We refer to a bipartite graph where both partitions have size $k$ and each vertex has a degree of $k-1$ as a \emph{$k$-crown graph}.
\end{definition}

\Cref{fig:k_n} shows two example instances of a $k$-crown graph for $k=3$ and $k=4$. Note that the $3$-crown graph is isomorphic to the $C_6$.

    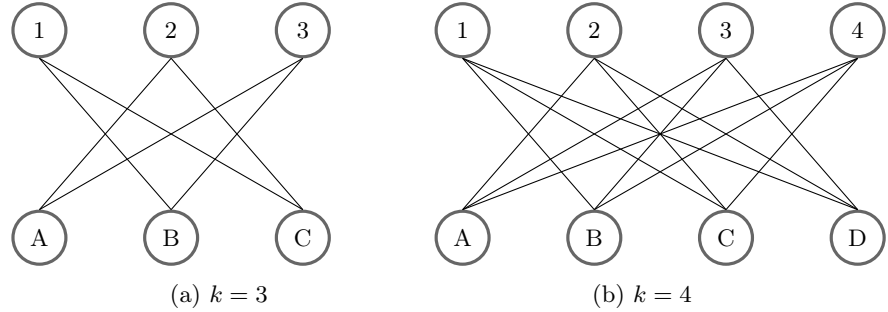
\begin{figure} [hbt!] 
        \centering
        \begin{subfigure}[b]{0.45\textwidth}
        \captionsetup{justification=centering}
        \begin{tikzpicture}[
    greennode/.style={circle, draw=black!60, fill=white!5, very thick, minimum size=7mm},
    rednode/.style={circle, draw=black!60, fill=white!5, very thick, minimum size=7mm},
    bluenode/.style={circle, draw=black!60, fill=white!5, very thick, minimum size=7mm},
    ]
    \node[greennode]      (green1)                              {1};
    \node[greennode]        (greenA)       [below=2cm of green1] {A};
    \node[rednode]        (red2)       [right=of green1] {2};
    \node[rednode]        (redB)       [right=of greenA] {B};
    \node[bluenode]        (blue3)       [right=of red2] {3};
    \node[bluenode]        (blueC)       [right=of redB] {C};
    
    \draw[-] (green1.south) -- (redB.north);
    \draw[-] (green1.south) -- (blueC.north);
    \draw[-] (red2.south) -- (greenA.north);
    \draw[-] (red2.south) -- (blueC.north);
    \draw[-] (blue3.south) -- (greenA.north);
    \draw[-] (blue3.south) -- (redB.north);
    
    \end{tikzpicture}
        \caption{$k=3$}
        \end{subfigure}
            \begin{subfigure}[b]{0.45\textwidth}
            \captionsetup{justification=centering}
        \begin{tikzpicture}[
    greennode/.style={circle, draw=black!60, fill=white!5, very thick, minimum size=7mm},
    rednode/.style={circle, draw=black!60, fill=white!5, very thick, minimum size=7mm},
    bluenode/.style={circle, draw=black!60, fill=white!5, very thick, minimum size=7mm},
    whitenode/.style={circle, draw=black!60, fill=white!5, very thick, minimum size=7mm},
    ]
    \node[greennode]      (green1)                              {1};
    \node[greennode]        (greenA)       [below=2cm of green1] {A};
    \node[rednode]        (red2)       [right=of green1] {2};
    \node[rednode]        (redB)       [right=of greenA] {B};
    \node[bluenode]        (blue3)       [right=of red2] {3};
    \node[bluenode]        (blueC)       [right=of redB] {C};
    \node[whitenode]      (white4)            [right=of blue3] {4};
    \node[whitenode]        (whiteD)       [right=of blueC] {D};
    
    \draw[-] (green1.south) -- (redB.north);
    \draw[-] (green1.south) -- (blueC.north);
    \draw[-] (green1.south) -- (whiteD.north);
    \draw[-] (red2.south) -- (greenA.north);
    \draw[-] (red2.south) -- (blueC.north);
    \draw[-] (red2.south) -- (whiteD.north);
    \draw[-] (blue3.south) -- (greenA.north);
    \draw[-] (blue3.south) -- (redB.north);
    \draw[-] (blue3.south) -- (whiteD.north);
    \draw[-] (white4.south) -- (greenA.north);
    \draw[-] (white4.south) -- (redB.north);
    \draw[-] (white4.south) -- (blueC.north);
    
    \end{tikzpicture}
        \caption{$k=4$}
        \end{subfigure}
            \caption{Two example instances of a $k$-crown graph for $k=3$ and $k=4$.}
    \label{fig:k_n}
    \end{figure}

We further introduce the term \emph{opposing pair} to refer to two vertices in opposite partitions that are not connected by an edge.

\begin{definition}[Opposing Pair] \label{def:opp_pair}
    Let $G=(V,E)$ be a bipartite graph with partitions $U\uplus W$. Let $u\in U$ and $w\in W$. Then, we call $u$ and $v$ an \emph{opposing pair} if $\{u, v\} \notin E$.
\end{definition}

We now prove that the fitness landscape for any bipartite graph $G = (V, E)$ contains NGLOPs at $k$ colors with $k >2$ and is therefore multimodal, provided that $G$ contains a $k$-crown graph as an induced subgraph and certain criteria regarding the structure of the rest of the graph are met. Intuitively, the fitness landscape for the $k$-crown graph contains an NGLO (which we prove in \Cref{lem:k_k}), which leads to an NGLOP in the fitness landscape for $G$ unless a part of $G$'s structure makes the necessary coloring of the $k$-crown graph impossible.

\begin{theorem} \label{the:crown}
    Let $k>2$. Let $G=(V, E)$ be a bipartite graph such that 
    \begin{enumerate}
    \item $G$ contains a $k$-crown graph as an induced subgraph,
    \item no vertex outside of the $k$-crown graph is connected to all vertices in one of the $k$-crown graph's partitions and 
    \item for all $u, v \in V$ and $a, b$ as vertices in the $k$-crown graph, it holds that if $(u,a), (v, b)$ and $(a,b)$ are opposing pairs, then so is $(u,v)$.
\end{enumerate}
    Then, the fitness landscape $(\CalX_n, \flip, \fCol_G)$ is multimodal with NGLOPs at $k$ colors.
\end{theorem}

To prove this, we first show that the fitness landscape for the $k$-crown graph contains an NGLO at $k$ colors and then conclude that this causes an NGLOP in the fitness landscape for $G$. We denote by $V_k \subset V$ the vertices in the $k$-crown graph.

\begin{lemma} \label{lem:k_k}
    The fitness landscape $(\CalX_n, \flip, \fCol_G)$ for the $k$-crown graph contains an NGLO at $k$ colors.
 \end{lemma}
 \begin{proof}
     Let $x$ be a feasible $k$-coloring. Then a feasible coloring $x$ is an NGLO if, for each vertex, it holds that all of its neighbors have distinct colors. \Cref{fig:local_opt} gives examples of such colorings for $k=3$ and the $k=4$.

    \begin{figure}[htb!]   
    \centering
    \begin{subfigure}[b]{0.45\textwidth}
    \captionsetup{justification=centering}
    \begin{tikzpicture}[
greennode/.style={circle, draw=green!60, fill=green!5, very thick, minimum size=7mm},
rednode/.style={rectangle, draw=red!60, fill=red!5, very thick, minimum size=6mm},
bluenode/.style={rectangle, draw=blue!60, fill=blue!5, very thick, minimum size=6mm, rounded corners},
]
\node[greennode]      (green1)                              {1};
\node[greennode]        (greenA)       [below=2cm of green1] {A};
\node[rednode]        (red2)       [right= 1.25cm of green1] {2};
\node[rednode]        (redB)       [right=1.2cm of greenA] {B};
\node[bluenode]        (blue3)       [right=1.25cm of red2] {3};
\node[bluenode]        (blueC)       [right=1.2cm of redB] {C};

\draw[-] (green1.south) -- (redB.north);
\draw[-] (green1.south) -- (blueC.north);
\draw[-] (red2.south) -- (greenA.north);
\draw[-] (red2.south) -- (blueC.north);
\draw[-] (blue3.south) -- (greenA.north);
\draw[-] (blue3.south) -- (redB.north);

\end{tikzpicture}
    \caption{$k=3$}
    \end{subfigure}
        \begin{subfigure}[b]{0.45\textwidth}
        \captionsetup{justification=centering}
    \begin{tikzpicture}[
greennode/.style={circle, draw=green!60, fill=green!5, very thick, minimum size=7mm},
rednode/.style={rectangle, draw=red!60, fill=red!5, very thick, minimum size=6mm},
bluenode/.style={rectangle, draw=blue!60, fill=blue!5, very thick, minimum size=6mm, rounded corners},
whitenode/.style={circle, draw=black!60, fill=white!5, very thick, minimum size=7mm},
]
\node[greennode]      (green1)                              {1};
\node[greennode]        (greenA)       [below=2cm of green1] {A};
\node[rednode]        (red2)       [right=of green1] {2};
\node[rednode]        (redB)       [right=of greenA] {B};
\node[bluenode]        (blue3)       [right=of red2] {3};
\node[bluenode]        (blueC)       [right=of redB] {C};
\node[whitenode]      (white4)            [right=of blue3] {4};
\node[whitenode]        (whiteD)       [right=0.875cm of blueC] {D};

\draw[-] (green1.south) -- (redB.north);
\draw[-] (green1.south) -- (blueC.north);
\draw[-] (green1.south) -- (whiteD.north);
\draw[-] (red2.south) -- (greenA.north);
\draw[-] (red2.south) -- (blueC.north);
\draw[-] (red2.south) -- (whiteD.north);
\draw[-] (blue3.south) -- (greenA.north);
\draw[-] (blue3.south) -- (redB.north);
\draw[-] (blue3.south) -- (whiteD.north);
\draw[-] (white4.south) -- (greenA.north);
\draw[-] (white4.south) -- (redB.north);
\draw[-] (white4.south) -- (blueC.north);

\end{tikzpicture}
    \caption{$k=4$}
    \end{subfigure}
        \caption{Two example instances of a $k$-crown graph for $k=3$ and $k=4$. As any given vertex is connected to at least one instance of each other color, any flip would lead to an infeasible coloring.}
\label{fig:local_opt}
\end{figure}
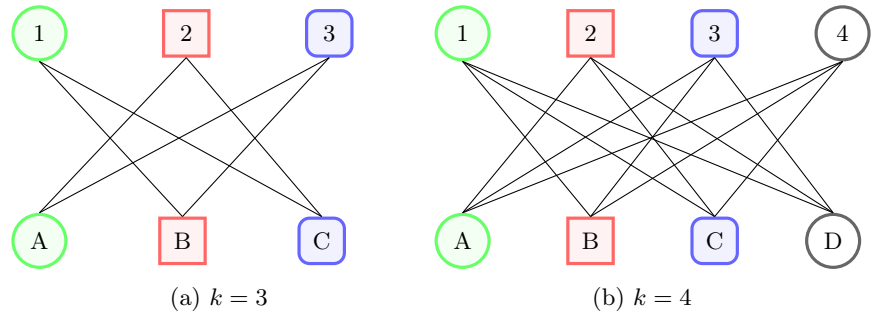

Assume $x$ is such a coloring, meaning that each of the $k$ colors appears exactly once in each partition. Let $v\in V_k$. We know that $v$ has $k-1$ neighbors and therefore $k-1$ neighboring colors. This means that, when changing the color of $v$, there are two possibilities for neighboring colorings $y$:

\begin{enumerate}
    \item $y_v$ is one of the $k-1$ colors appearing in $x$ other than $x_v$. That color must also be one of the $k-1$ colors neighboring $v$. Then, we gain one monochromatic edge, worsening the fitness in comparison to $x$.
    \item $y_v$ is a color that does not appear in $x$. In that case, the number of colors in use in $y$ is greater than that in $x$, worsening the fitness.
\end{enumerate}

Therefore, any possible neighbors of $x$ have worse fitness, making $x$ a local optimum. We also know that $x$ is not a global optimum, as we have $k>2$ colors in use, but know the $k$-crown graph to be $2$-colorable. Therefore, $x$ is an NGLO.
 \end{proof}

 Now, we use this lemma to prove \Cref{the:crown}.
 \begin{proof} [Proof of \Cref{the:crown}]
     From \Cref{lem:k_k}, we know that we can $k$-color the $k$-crown graph in such a way that the coloring is an NGLO by coloring each of the $k$ vertices in each partition in a distinct color (while maintaining a feasible coloring). It remains to prove that such a coloring is always possible in $G$. 

     Now, assume that there are $k$ colors in use in $x$. Further, the $k$-crown graph is colored with each color used only once per partition. Let $x$ be a such coloring. Now, we must show that, for all vertices $v$ outside of the $k$-crown graph, there exists a color $x_v$ such that $x$ remains feasible.

     Let $v \in V \setminus V_k$. Let $a \in V_k$ be a vertex in the opposite partition as $v$ such that $(a,v)$ is an opposing pair. We know that such a vertex must exist, as no vertex in $G$ is connected to all vertices in one partition of the $k$-crown graph. Then, let $x_v=x_a$. 
     
     Now, let $u\in V$ be a vertex such that $\{u,v\} \in E$ (meaning $(u,v)$ is not an opposing pair) and let $b \in V_k$ such that $(u,b)$ is an opposing pair. Then we know that $(a,b)$ cannot be an opposing pair, meaning  $\{a,b\} \in E$ must hold and therefore $x_a \ne x_b$, which also implies $x_v \ne x_u$. 

     This means that for any vertex $v$, there exists a color such that no monochromatic edge is created. Therefore, we can construct the described coloring without creating monochromatic edges, proving that an NGLOP at $k$ colors exists.
 \end{proof}

Just as for the Dominating Set problem, we derive further results by turning to results from reconfiguration theory. In the following, for a given graph $G$, we use $\CalF_G^\ell$ to denote the set of all proper $\ell$-colorings of $G$.

\citeauthor{bonamy_reconfiguration_2014} \cite{bonamy_reconfiguration_2014} consider changing the color of a single vertex as a modification rule; this corresponds to the \flipN. They give a general graph class, called $k$-color-dense graphs, and show that, for any such graph $G$ and $\ell \geq k+1$, $(\CalF_G^\ell,\flip)$ is connected under this modification rule. As example subclasses of this general graph class they give the class of all chordal graphs and all chordal bipartite graphs.

We give the following general theorem to turn such connectedness results into landscape results.

\begin{theorem}\label{the:coloring_connected}
    Let $G$ be a graph with chromatic number $k$ such that for every $l\ge k+1$, $(\CalF_G^\ell,\flip)$ is connected. Then the fitness landscape $(\CalX_n,\flip, \fCol_G)$ is equimodal.
\end{theorem}

\begin{proof}
    We first show that, for any $j \geq k$ and any colorings $x \in \CalF_G^j$ and $y \in \CalF_G^{j+1}$, we have $y \transReachPlus x$. Let $j \geq k$, $x \in \CalF_G^j$ and $y \in \CalF_G^{j+1}$ be given. Let $z$ be any coloring derived from $x$ by recoloring a single vertex with a color not used in $x$. Clearly, $z$ is feasible, so $z \in \CalF_G^{j+1}$, and $z \reachPlus x$. Using the assumption of the theorem, we now get
    $
    y \transReachPlus z \reachPlus x
    $
    as desired. 
    Let now $x \in \CalF_G$ be a coloring using more than $k$ colors and $y$ a coloring using exactly $k$ colors. Then, with a straightforward induction using the claim we just showed, we have $x \transReachPlus y$. In particular, $(\CalX_n,\flip, \fCol_G)$ has no NGLOPs, as claimed.
\end{proof}

Using this as well as results given by \citeauthor{bonamy_reconfiguration_2014} \cite{bonamy_reconfiguration_2014}, we can now derive the following theorems for chordal graphs and chordal bipartite graphs; we give the well-known definition for completeness.

\begin{definition}[Chordal Graph]
    A \emph{chordal graph} is a graph with no induced cycle of length more than $3$.
\end{definition}

\begin{theorem} \label{the:coloring_chordal}
    Let $G$ be a chordal graph with chromatic number $k$. Then the fitness landscape $(\CalX_n,\flip, \fCol_G)$ is equimodal.
\end{theorem}
\begin{proof}
    From \cite[Theorem 6]{bonamy_reconfiguration_2014} we have that any $k'$-colorable chordal graph is $k'$-color-dense. Therefore, $G$ is specifically $k$-color dense. It then follows from \cite[Theorem 2]{bonamy_reconfiguration_2014} that, for all $l \ge k+1$, $\CalF_G^l$ is connected. Therefore, by \Cref{the:coloring_connected}, it follows that the fitness landscape for $G$ is equimodal for the \flipN.
\end{proof}

As the following example demonstrates, there are chordal graphs that have multiple distinct colorings, even taking into account relabeling of colors and graph automorphisms. Therefore, this fitness landscape is not plateau-unimodal.

\begin{theorem}
    There is a chordal graph $G$ such that the landscape $(\CalX_n,\flip, \fCol_G)$ is not plateau-unimodal.
\end{theorem}
\begin{proof}
    See \Cref{fig:counterexampleColoringUnimodal} for two different global optima where for both the green vertex and its four neighbors cannot be recolored in any recoloring sequence without violating feasibility. Thus, they cannot be transformed into each other. Note that the further vertices ensure that the global optima are not isomorphic (in terms of isomorphisms of the graph).
\end{proof}

\begin{figure}[htb!]   
    \centering
    \begin{tikzpicture}[
        greennode/.style={circle, draw=green!60, fill=green!5, very thick, minimum size=7mm},
        rednode/.style={rectangle, draw=red!60, fill=red!5, very thick, minimum size=6mm},
        bluenode/.style={rectangle, draw=blue!60, fill=blue!5, very thick, minimum size=6mm, rounded corners},
    ]
        \node[bluenode] (V1) at (0, 4) {B};
        \node[rednode] (V2) at (2, 4) {R}
            edge (V1);
        \node[bluenode] (V3) at (4, 4) {B}
            edge (V2);
        \node[rednode] (V4) at (4, 6) {R}
            edge (V3);

        \node[greennode] (U) at (6, 5) {G}
            edge (V3) edge (V4);

        \node[rednode] (W1) at (8, 6) {R}
            edge (U);
        \node[bluenode] (W2) at (8, 4) {B}
            edge (U) edge (W1);
        \node[rednode] (W3) at (10, 4) {R}
            edge (W2);
    
        \node[bluenode] (v1) at (0, 0) {B};
        \node[rednode] (v2) at (2, 0) {R}
            edge (v1);
        \node[bluenode] (v3) at (4, 0) {B}
            edge (v2);
        \node[rednode] (v4) at (4, 2) {R}
            edge (v3);

        \node[greennode] (u) at (6, 1) {G}
            edge (v3) edge (v4);

        \node[bluenode] (w1) at (8, 2) {B}
            edge (u);
        \node[rednode] (w2) at (8, 0) {R}
            edge (u) edge (w1);
        \node[bluenode] (w3) at (10, 0) {B}
            edge (w2);
        
    \end{tikzpicture}
    \caption{A chordal graph with two different optimal colorings (even up to symmetry).}
    \label{fig:counterexampleColoringUnimodal}
\end{figure}
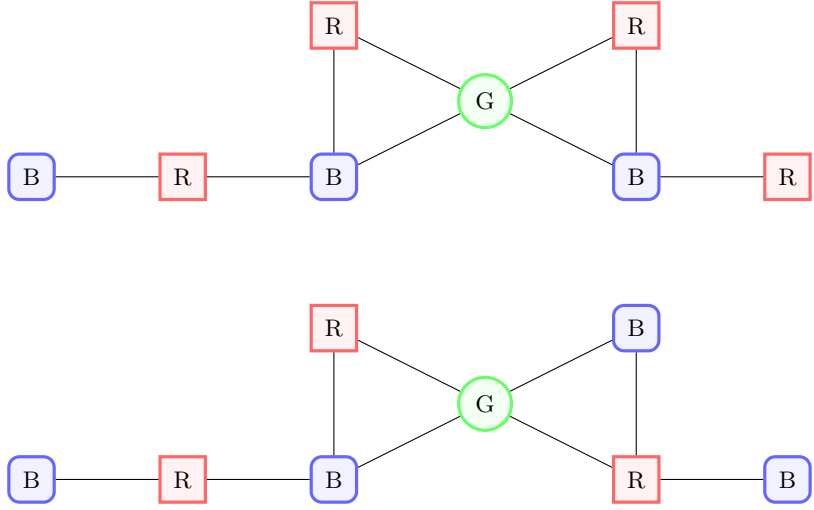

\begin{definition}[Chordal Bipartite Graph]
    A \emph{chordal bipartite graph} is a bipartite graph with no induced cycle of length greater than $4$.
\end{definition}

\begin{theorem}\label{the:coloring_chordal_bipartite}
    Let $G$ be a chordal bipartite graph. Then the fitness landscape $(\CalX_n,\flip, \fCol_G)$ is equimodal, with all global optima being isomorphic under permuting the colors.
\end{theorem}
\begin{proof}
    Per a combination of Theorems 8 (semi-false graphs are a proper superclass of chordal bipartite graphs), 9 (semi-false graphs are 2-colorable) and 10 (semi-false graphs are 2-color-dense) given by \citeauthor{bonamy_reconfiguration_2014} \cite{bonamy_reconfiguration_2014}, any chordal bipartite graph is $2$-color-dense. Therefore, $G$ is $2$-color dense. It then follows from Theorem 2 by \citeauthor{bonamy_reconfiguration_2014} \cite{bonamy_reconfiguration_2014} that, for all $l \geq 3$, $\CalF_G^l$ is connected. Therefore, by \Cref{the:coloring_connected}, it follows that the fitness landscape for $G$ is equimodal for the \flipN.
    As for all connected bipartite graphs, due to the fact that $2$-colorings are unique up to permuting colors, equimodality implies unimodality for chordal bipartite graphs.
\end{proof}

\subsection{\flipSwapN}

We now turn to the \flipSwapN. We show that $(\CalX_n, \flipSwap, \fCol_G)$ is unimodal for both the $C_{6k}$ and crown graphs, but multimodal for a graph class we introduce called \emph{spoked $C_{12k}$}, which is a variant of the $C_{12k}$ with certain chords added. 

\begin{theorem}\label{the:c6kSwap}
    Let $G=(V,E)$ be a $C_{6k}$. Then the landscape $(\CalX_n, \flipSwap, \fCol_G)$ is unimodal.
\end{theorem}

\begin{proof}
    Let $G=(V,E)$ be a $C_{6k}$ and $x$ any coloring using at least three colors. We show that no NGLOPs exist by proving the existence of a positive path from $x$ to a $2$-coloring.

    First, we show that a $3$-coloring $y$ exists such that $x \transReachPlus y$. For this, we fix three colors $C$ currently in use in $y$ and an arbitrary order on these colors. Now we recolor all vertices one by one by assigning them the minimal color from $C$ not present in their neighbors.

    It remains to be shown that there exists a $2$-coloring $z$ such that $y \transReachPlus z$. For this, we fix two colors in $y$ as final colors, leaving one non-final color. We call the non-final color yellow.
    
    \begin{lemma}
        Let $y$ be a $3$-coloring on the $C_{6k}$. Then there exists a feasible coloring $z$ such that $z$ contains one fewer instance of yellow than $y$.
    \end{lemma}
    \begin{proof}
        Let $v \in V$ be any yellow vertex. We now differentiate several cases. 

        \emph{Case 1: Both neighbors of $v$ are colored in the same color.}

        In that case, it is possible to flip $v$ to the non-neighboring final color and thereby remove one instance of yellow.

        \emph{Case 2: Both neighbors of $v$ are colored in different colors.}

        First, we note that this case implies that at least two yellow vertices in $y$, as otherwise $G-v$ is a path of odd length colored in two colors, so start and end vertex would be colored in the same color, contradicting the assumption of the case.

        Therefore, we must have at least two yellow vertices. Let $u \ne v$ be another yellow vertex such that there is no vertex of the same color closer to $v$. Now, we prove by induction that for any number of vertices between $u$ and $v$, there exists a coloring $y'$ such that there is one fewer instance of the non-final color in $y'$ than in $y$. 

        We now fix a segment of the circle starting and ending with a yellow vertex, but in between there is no yellow vertex on the segment. We can now step-by-step swap one of the yellow vertices towards the other vertices and terminate when such a swap would create a monochromatic edge. If the monochromatic edge would appear because both neighbors of the yellow vertex to be swapped have the same color, then we are done by \emph{Case 1}. Otherwise it terminates because otherwise the two yellow vertices would be neighboring. Now we are either done by \emph{Case 1} or we can recolor the single vertex between the two yellow vertices to the other final color and are done by \emph{Case 1}. 

    \end{proof}
    Now, we can use this to prove the theorem: We know that, while there are still instances of the final color left in the coloring, we can apply the above lemma and arrive at a positive path to a coloring with one fewer instance of the non-final color. By repeated application, we receive a coloring $z$ such that $y \transReachPlus z$ and $z$ is a $2$-coloring using only the two final colors. With this, we have proven equimodality and, by \Cref{prop:uni_equi}, also unimodality.
\end{proof}

\begin{theorem}\label{the:crownSwap}
    Let $G=(V,E)$ be a $k$-crown graph. Then, the fitness landscape $(\CalX_n, \flipSwap, \fCol_G)$ is unimodal.
\end{theorem}
\begin{proof}
    Let $G=(V,E)$ be a $k$-crown graph and $x$ any coloring using at least three colors. We show that no NGLOPs exist by proving the existence of a positive path from $x$ to a $2$-coloring.

    First, we show that if $x$ contains more than $c$ colors, there exists a positive path to a $k$-coloring. Then, we prove that for any coloring $y$ using $k$ or fewer colors, a positive path to a $2$-coloring exists.

    For the first part, fix any $k$ colors $C$ in use in $x$. Recolor, one by one, each vertex with a color from $C$ which none of its $k-1$ neighbors has. Note that we can never lose a color from $C$, so the fitness will never increase, since coloring a vertex with a color from $C$ will never introduce a new color. Thus, for the resulting coloring $y$, we have $x \transReachPlus y$.

    We now differentiate two cases.

    \emph{Case 1: There is one color $c$ such that it only appears in one partition of the crown graph.}

    Call that partition $U$. Then a positive path exists from $y$ to a coloring $z$ such that all vertices in $U$ are colored in $c$ in $z$. Further, from $z$, a positive path exists to a coloring $z$ such that all vertices in $U$ remain colored in $c$, and all vertices in the opposite partition are colored in one color. The proof for this runs analogously to the one given for \Cref{the:universal}.

    \emph{Case 2: All colors appear in both partitions.}

    We first note that this implies $y$ being a $k$-coloring, with each color appearing exactly twice (once in each partition), since otherwise a color has to be present twice in a partition, so at least one of them will be adjacent to a vertex of the same color in the other partition. 

    Let $c$ be any color in use in $y$ and $u,v \in V$ the two vertices colored in $c$ in $y$ and let $w \in V$ be any other vertex in the same partition as $v$. Define a coloring $z$ derived from $y$ by swapping the colors of $u$ and $w$. 
    We know $z$ to be feasible, as $u$ and $w$ were the only instances of their respective colors in their partitions, so the swap did not create monochromatic edges. Further, the swap did not change the number of colors used. Therefore, $y \reachPlus z$.

    Now, we know $z$ to be a coloring where all instances of the color $c$ appear in the same partition. Therefore, by \emph{Case 1}, there exists a $2$-coloring $a$ such that $z \transReachPlus a$.

    With this, we have proven that $k$-crown graphs contain no NGLOPs for the \flipSwapN, making the corresponding fitness landscape equimodal. As crown graphs are bipartite, it is also unimodal by \Cref{prop:uni_equi}.
\end{proof}

\newcommand{\drawSpokedC}{

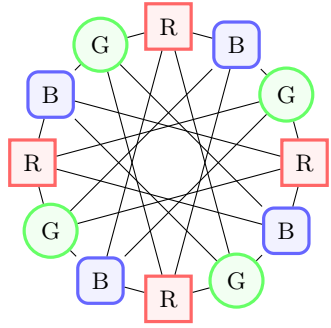
\begin{figure}[htb!]
    \centering
    \begin{tikzpicture}[scale=0.6,
        greennode/.style={circle, draw=green!60, fill=green!5, very thick, minimum size=7mm},
        rednode/.style={rectangle, draw=red!60, fill=red!5, very thick, minimum size=6mm},
        bluenode/.style={rectangle, draw=blue!60, fill=blue!5, very thick, minimum size=6mm, rounded corners},
    ]
        \node[rednode] (v0) at (canvas polar cs:angle=0,radius=3cm) {R};
        \node[greennode] (v1) at (canvas polar cs:angle=30,radius=3cm) {G};
        \node[bluenode] (v2) at (canvas polar cs:angle=60,radius=3cm) {B};
        \node[rednode] (v3) at (canvas polar cs:angle=90,radius=3cm) {R};
        \node[greennode] (v4) at (canvas polar cs:angle=120,radius=3cm) {G};
        \node[bluenode] (v5) at (canvas polar cs:angle=150,radius=3cm) {B};
        \node[rednode] (v6) at (canvas polar cs:angle=180,radius=3cm) {R};
        \node[greennode] (v7) at (canvas polar cs:angle=210,radius=3cm) {G};
        \node[bluenode] (v8) at (canvas polar cs:angle=240,radius=3cm) {B};
        \node[rednode] (v9) at (canvas polar cs:angle=270,radius=3cm) {R};
        \node[greennode] (v10) at (canvas polar cs:angle=300,radius=3cm) {G};
        \node[bluenode] (v11) at (canvas polar cs:angle=330,radius=3cm) {B};

        \draw (v0) to (v1);
        \draw (v1) to (v2);
        \draw (v2) to (v3);
        \draw (v3) to (v4);
        \draw (v4) to (v5);
        \draw (v5) to (v6);
        \draw (v6) to (v7);
        \draw (v7) to (v8);
        \draw (v8) to (v9);
        \draw (v9) to (v10);
        \draw (v10) to (v11);
        \draw (v11) to (v0);

        \draw (v0) to (v5);
        \draw (v0) to (v7);
        
        \draw (v1) to (v6);
        \draw (v1) to (v8);
        
        \draw (v2) to (v7);
        \draw (v2) to (v9);
        
        \draw (v3) to (v8);
        \draw (v3) to (v10);
        
        \draw (v4) to (v9);
        \draw (v4) to (v11);
        
        \draw (v5) to (v10);
        
        \draw (v6) to (v11);
    \end{tikzpicture}
    \caption{The spoked $C_{12}$.}
    \label{fig:spoked_c12}
\end{figure}
}

Finally, we give an example of a class of graphs that induces a multimodal landscape on the \flipSwapN to demonstrate that flipping and swapping cannot color even some simple graphs.

\begin{definition}[Spoked $C_{12k}$]
    Let $G$ be the cycle graph $C_{12k}$ with additional edges as follows. Let $V = \{0, \ldots 12k-1\}$ be the vertices of $G$ in order around the cycle. For any vertex $i$ we call $(i + 6k) \bmod 12k$ its opposite vertex. For any vertex $v$ we then add an edge from $v$ to both of its opposite vertex's neighbors on the cycle.
\end{definition}

\drawSpokedC

Note that the spoked $C_{12k}$ is bipartite as $C_{12k}$ is bipartite and we only added edges between different partitions. \Cref{fig:spoked_c12} shows the spoked $C_{12}$.
We now state that any spoked $C_{12k}$ induces a multimodal landscape on the \flipSwapN.

\begin{theorem}\label{the:spoked12k}
    Let $G$ be a spoked $C_{12k}$. Then $(\CalX_n, \flipSwap, \fCol_G)$ is multimodal. 
\end{theorem}

\begin{proof}
    We state a feasible $3$-coloring of $G$ and then prove that it is an NGLO in the landscape $(\CalX_n, \flipSwap, \fCol_G)$.
    Let the vertices of the cycle be colored red, green and blue cyclically. This means that for any vertices $i, j$, if $i \equiv j \pmod 3$, then $x_i = x_j$. For an example, see Figure \ref{fig:spoked_c12}. Note that since $12k$ is divisible by $3$, each color appears equally often.

    To show that this coloring is feasible, note that each vertex $i$ four neighbors: its two neighbors along the cycle as well as the vertices neighboring its opposite vertex along the cycle. We know that its opposite vertex along the cycle has the same color as itself, as $i \equiv i+6k \pmod 3$. Likewise, its two neighbors along the cycle have the same two colors as the two neighbors on the cycle of its opposite vertex. We know those colors to be different from $x_i$.

    We now prove that this coloring is a local optimum. Flipping any vertex clearly either introduces a fourth color, thus worsening the fitness, or gives it the same color as a neighbor of itself, as each vertex is adjacent to both remaining colors, rendering the coloring infeasible.

    To show that any swap also worsens the fitness, let $v$ be a vertex. Since $v$ has four neighbors, two of each remaining color, swapping with either of them creates a monochromatic edge with the other one, thereby rendering the coloring infeasible.
    As we have stated a coloring which is an NGLO at three colors, we have proved the landscape to be multimodal.
\end{proof}

\printbibliography

\appendix

\end{document}